\documentclass{article}

\usepackage{arxiv}

\usepackage[utf8]{inputenc} 
\usepackage[T1]{fontenc}    
\usepackage{hyperref}       
\usepackage{url}            
\usepackage{booktabs}       
\usepackage{amsfonts}       
\usepackage{amsmath,amsthm,amssymb}
\usepackage{nicefrac}       
\usepackage{microtype}      
\usepackage{cleveref}       
\usepackage{lipsum}         
\usepackage{graphicx}
\usepackage[square,numbers]{natbib}
\usepackage{doi}

\usepackage{algpseudocode}
\usepackage{algorithm}
\usepackage{subcaption}
\usepackage{comment}
\usepackage{bm}
\usepackage{chngcntr}
\usepackage{bibunits}
\usepackage{multirow}

\newtheorem{theorem}{Theorem}%
\newcommand{\E}[1]{\mathbb{E}\left[#1\right]}
\newcommand{\prob}[1]{\text{Pr}\left\{#1\right\}}
\newcommand{\cond}{\,\bigg\vert\,}

\newtheorem{lemma}[theorem]{Lemma}
\counterwithin*{theorem}{section}

\title{Online Learning of Network Bottlenecks via Minimax Paths}

\hypersetup{
pdftitle={Online Learning of Network Bottlenecks via Minimax Paths},
pdfsubject={cs.LG, stat.ML},
pdfauthor={Niklas {\AA}kerblom, Fazeleh Sadat Hoseini, Morteza Haghir Chehreghani},
pdfkeywords={Online Learning, Combinatorial Semi-bandit,Thompson Sampling, Bottleneck Identification},
}

\begin{document}

\author{Niklas {\AA}kerblom \\
	Volvo Car Corporation \\
	Gothenburg, Sweden \\[3pt] 
	Department of Computer Science and Engineering\\
	Chalmers University of Technology\\
	Gothenburg, Sweden \\
	\texttt{niklas.akerblom@chalmers.se} \\
	\And
	Fazeleh Sadat~Hoseini \\
	Department of Computer Science and Engineering\\
	Chalmers University of Technology\\
	Gothenburg, Sweden \\
	\texttt{fazeleh@chalmers.se} \\
    \And
	Morteza Haghir~Chehreghani \\
	Department of Computer Science and Engineering\\
	Chalmers University of Technology\\
	Gothenburg, Sweden \\
	\texttt{morteza.chehreghani@chalmers.se} \\
}

\maketitle

\begin{abstract}
In this paper, we study bottleneck identification in networks via extracting minimax paths. Many real-world networks have stochastic weights for which full knowledge is not available in advance. Therefore, we model this task as a combinatorial semi-bandit problem to which we apply a combinatorial version of Thompson Sampling and establish an upper bound on the corresponding Bayesian regret. Due to the computational intractability of the problem, we then devise an alternative problem formulation which approximates the original objective. Finally, we experimentally evaluate the performance of Thompson Sampling with the approximate formulation on real-world directed and undirected networks.
\end{abstract}


\section{Introduction}
Bottleneck identification constitutes an important task in network analysis, with applications including transportation planning and management \citep{berman1987optimal}, routing in computer networks \citep{shacham1992multicast} and various bicriterion path problems \citep{hansen80}.
The path-specific bottleneck, on a path between a source and a target node in a network, is defined as the edge with a maximal cost or weight according to some criterion such as transfer time, load, commute time, distance, etc.
The goal of bottleneck identification and avoidance is then to find a path whose bottleneck is minimal. Thus, one may model bottleneck identification as the problem of computing the minimax edge over the given network/graph, to obtain an edge with a minimal largest gap between the source and target nodes. Equivalently, it can be formulated as a widest path problem or maximum capacity path problem \citep{pollack1960letter} where the edge weights have been negated.

In transportation, identifying bottlenecks is important for city governments and traffic managers to monitor and improve overall performance. For example, when driving between two locations, identifying a bottleneck means finding the section of road that slows down the traffic between two locations. 
Consider a social network where there is an information flow between two nodes (source and destination). Then we can define the bottleneck as a link in all the paths between these two nodes that has the weakest connection and causes the information flow to fail. So, bottlenecks here are related to the possibility of information loss.

The aforementioned formulations assume that the network or the graph is fully specified, i.e., that all the edge weights are fully known. However, in practice, the edge weights might not be known in advance or they might include some inherent uncertainty. In this paper, we tackle such situations by developing an online learning framework to learn the edge weight distributions of the underlying network, while solving the bottleneck identification problem for different problem instances.

For example, in the transportation scenario, city governments often have access to fleets of vehicles utilized for various municipal services. These may be used to sequentially and continuously to gain knowledge about traffic flow from the environment, while it is still desirable to avoid causing unnecessary inconvenience and stress \cite{hennessy1999} to the employees operating the vehicles by excessively exploring congested paths. If care is taken to spread the costs over time, exploration may be performed continuously without having a specific end time known in advance (i.e., the \emph{time horizon} of the sequential decision making problem).

For this purpose, we view this as a multi-armed bandit (MAB) problem and focus on Thompson Sampling (TS) \citep{thompson1933}, a method that suits probabilistic online learning well. Thompson Sampling is an early Bayesian method for addressing the trade-off between exploration and exploitation in sequential decision making problems. It balances these by randomly sampling available actions according to their posterior probability of being optimal, given prior beliefs and observations from previously selected actions. An action is more likely to be sampled if the posterior distribution over the expected reward of that action has high uncertainty (exploration) or high mean (exploitation). 

The method has only recently been thoroughly evaluated through experimental studies \citep{ChapelleL11, GraepelCBH10} and theoretical analyses \citep{KaufmannKM12, agrawal2012analysis, russo2014learning}, where it has been shown to be asymptotically optimal in the sense that it matches well-known lower bounds of these types of problems \citep{lai1985asymptotically}. Furthermore, the algorithm does not assume knowledge of the time horizon, i.e., it is an \emph{anytime} algorithm.

Among many other problem settings, Thompson Sampling has been adapted to online versions of combinatorial optimization problems with retained theoretical guarantees \citep{wang2018thompson}, where one application is to find shortest paths in graphs \citep{liu2012adaptive, gai2012combinatorial, zou2014online, ijcai2020-0284}. 

Another commonly used method for these problems is Upper Confidence Bound (UCB) \citep{Auer02}, which utilizes optimism to balance exploration and exploitation. UCB has been adapted to combinatorial settings \citep{chen2013combinatorial}, and also exists in Bayesian variants \citep{kaufmann2012bayesian}. Recently, a variant of UCB has been studied for bottleneck avoidance problems in a combinatorial pure exploration setting \citep{du2021combinatorial}. They consider a different problem setting and method than those we present in this paper, though their bottleneck reward function is similar to the one we use in our approximation method. The main difference between their setting and the standard combinatorial semi-bandit setting in how agents interact with the environment, is that instead of being restricted to selecting sets of actions respecting combinatorial constraints, they allow agents to sequentially try individual arms to identify the best feasible solution to the combinatorial problem. This is not applicable to our setting, since we may not observe the feedback of individual edges without also traversing a path containing those edges, potentially incurring cost from some other edge on that path.

Moreover, the objective in a pure exploration problem is to find the best action as quickly as possible, with either a fixed time budget or confidence level, using agents dedicated for this task. While identifying the best path is desirable in our problem setting as well, we are specifically interested in the case where existing agents are utilized and where using them exclusively for exploration is too costly. For that reason, we focus on \emph{anytime} methods capable of distributing exploratory actions over time. 

In this paper, we model the online bottleneck identification task as a stochastic combinatorial semi-bandit problem, for which we develop a combinatorial variant of Thompson Sampling. We then derive an upper bound on the corresponding Bayesian regret that is tight up to a polylogarithmic factor, which is consistent with the existing lower bounds for combinatorial semi-bandit problems. We face the issue of computational intractability with the \emph{exact} problem formulation. We thus propose an \emph{approximation} scheme, along with a theoretical analysis of its properties. 
Finally, we experimentally investigate the performance of the proposed method on directed and undirected real-world networks from transport and collaboration domains.

\section{Bottleneck Identification Model}
In this section, we first introduce the bottleneck identification problem over a fixed network and then describe a probabilistic model to be used in stochastic and uncertain situations. 

\subsection{Bottleneck identification over a network}
We model a network by a graph $ G( V, E, w )$, where $V$ denotes the set of vertices (nodes) and each $e = (u,v) \in E$ indicates an edge between vertices $u$ and $v$ where $u,v \in V$ and $u \neq v$.
Moreover, $w: E\rightarrow \mathbb{R}$ is a weight function defined for each edge of the graph, where for convenience, we use $w_e$ to denote the weight of edge $e$.
If $G$ is directed, the pair $(u,v)$ is ordered, otherwise, it is not (i.e., $(u,v)\equiv (v,u)$ for undirected graphs).
A path $p$ from vertex $u$ (source) to vertex $v$ (target) over $G$ is a sequence of vertices $<v_1, v_2, \dots,v_{k-1}, v_k >$ where $v_1=u$, $v_k=v$ and $(v_i,v_{i+1}) \in E, \forall i \in [1, k-1]$. 
It can also be seen as a sequence of edges $< (v_1, v_2),(v_2,v_3), \dots,(v_{k-1}, v_k) >$.

As previously mentioned, a bottleneck on a path $p$ can be described as an edge with a maximal weight on that path. 
To find the smallest feasible bottleneck edge between the source node $u$ and the target node $v$, we consider all the paths between them. For each path, we pick an edge with a maximal weight, to obtain all path-specific bottleneck edges. We then identify the smallest path-specific bottleneck edge in order to find the best feasible bottleneck edge, i.e., such that bottleneck edges with higher weights are avoided.  

Therefore, given graph $G$, the bottleneck edge between $u \in V$ and $v \in V$ can be identified via extracting the minimax edge between them. With $P_{u,v}$ denoting the set of all possible paths from  $u$ to $v$ over $G$, the bottleneck weight (incurred by the bottleneck edge) can be computed by

\begin{equation}\label{eq:MM}
     b(u,v;G) = \min_{p \in P_{u,v}} \max_{e\in p} w_e .
\end{equation}

The quantity in Eq. \ref{eq:MM} satisfies the  (ultra) metric properties under some basic assumptions on the edge weights such as symmetry and nonnegativity. Hence, it is sometimes used as a proper distance measure to extract manifolds and elongated clusters in a non-parametric way \citep{Chehreghani20Minimx, KimC07}.

However, in our setting, such conditions  do not need to be fulfilled by the edge weights. In general,  we tolerate positive as well as negative edge weights, and  we assume the graph might be directed, i.e., the edge weights are not necessarily symmetric. Therefore, despite the absence of (ultra) metric properties, the concept of minimax edges is still relevant for bottleneck identification.

To compute the minimax edge, one does not need to investigate all possible paths between the source and target nodes, which might be computationally infeasible. As studied in \cite{Hu61}, minimax edges and paths over an arbitrary undirected  graph are equal to the minimax edges over any minimum spanning tree (MST) computed over that graph. This equivalence simplifies the calculation of minimax edges, as there is only one path between every two vertices over an MST, whose maximal edge weight yields the minimax edge, i.e., the desired bottleneck.

For directed graphs, an MST might not represent the minimax edges in a straightforward manner. Hence, we instead rely on a modification \citep{berman1987optimal} of Dijkstra's algorithm \citep{dijkstra1959note} to extract minimax paths rather than the shortest paths.

\subsection{Probabilistic model for bottleneck identification}

In this paper, we study bottleneck identification in uncertain and stochastic settings. Therefore, instead of considering the weights $w_e$ for $e \in E$ to be fixed, we view them as stochastic with fixed, albeit unknown, distribution parameters. Additionally, we assume that the weight of each edge follows a Gaussian distribution with known and finite variance. The Gaussian edge weight assumption is common for many important problem settings, like minimization of travel time \citep{seshadri2010algorithm} or energy consumption \citep{ijcai2020-0284} in road networks. Furthermore, we assume that all edge weights are mutually independent. Hence,

\begin{equation*}
    w_e \sim \mathcal{N}(\theta_e^*, \sigma_e^2) ,
\end{equation*}

where $\theta_e^*$ denotes the unknown mean of edge $e$, and $\sigma_e^2$ is the known variance. To reduce cumbersome notation in the proofs, since the variance is assumed to be finite, we let $\sigma_e^2 \leq 1$ (by scaling the edge weight distributions). However, we emphasize that we \emph{do not} assume that $w_e$ and $\theta_e^*$ are bounded or non-negative.

It is convenient to be able to make use of prior knowledge in online learning problems where the action space is large, which motivates a Bayesian approach where we assume that the unknown mean $\theta^*_e$ is sampled from a known prior distribution:

\begin{equation*}
    \theta_e^* \sim \mathcal{N}(\mu_{e,0}, \varsigma_{e,0}^2) .
\end{equation*}

We use a Gaussian prior for $\theta_e^*$ since it is conjugate to the Gaussian likelihood and allows for efficient recursive updates of posterior parameters upon a new weight observation $w_{e,t}$ at time $t$:

\begin{align}
    &\varsigma_{e,t+1}^2 \leftarrow \left(\frac{1}{\varsigma_{e,t}^2} + \frac{1}{\sigma^2_e}\right)^{-1} ,\label{eq:posterior_variance}\\ 
    &\mu_{e,t+1} \leftarrow \varsigma_{e,t+1}^2 \left(\frac{\mu_{e,t}}{\varsigma_{e,t}^2} + \frac{w_{e,t}}{\sigma^2_e}\right) . \label{eq:posterior_mean}
\end{align}

Since our long-term objective is to find a path which minimizes the \emph{expected} maximum edge weight along that path, we need a \emph{framework} to sequentially select paths to update these parameters and learn enough information about the edge weight distributions.

The assumptions in this section might seem restrictive, and indeed, when the edge weights represent e.g., traffic congestion in a road network, it is reasonable to believe that edges are not independent, especially for neighboring road segments. There are ways of extending this setting to capture such dependencies, while retaining similar regret guarantees for the studied methods. Such extensions include the contextual setting, where expected edge weights are assumed to follow parameterized functions of contextual features (e.g., time-of-day, local ambient temperature, precipitation) revealed to the agent in each time step, before each action is taken. We leave such extensions to future work, though we note that the proofs in this work may be extended in a straightforward manner, analogous to the analysis of linear contextual Thompson Sampling in \cite{russo2014learning}. Similarly, Thompson Sampling may be extended to the case where both the mean and variance are unknown, by assignment of a joint prior distribution over the parameters \cite{riquelme2018deep}. 

\section{Online Bottleneck Learning Framework}

Consider a stochastic combinatorial semi-bandit problem \citep{cesa2012combinatorial} with time horizon $T$, formulated as a problem of cost minimization rather than reward maximization. There is a set of base arms $\mathcal{A}$ (where we let $d := \vert \mathcal{A} \vert$) from which we may, at each time step $t \in [T]$, select a subset (or \emph{super arm}) $\bm{a}_t \subseteq \mathcal{A}$. The selection is further restricted such that $\bm{a}_t \in \mathcal{I} \subseteq 2^\mathcal{A}$, where $\mathcal{I}$ is called the set of \emph{feasible} super arms.

Upon selection of $\bm{a}_t$, the environment reveals a feedback $X_{i,t}$ drawn from some fixed and unknown distribution for each base arm $i \in \bm{a}_t$ (i.e., \emph{semi-bandit} feedback). Furthermore, we then receive a super arm cost from the environment, $c(\bm{a}_t) := \max_{i \in \bm{a}_t} X_{i,t}$, i.e., the maximum of all base arm feedback for the selected super arm and the current time step. The objective is to select super arms $\bm{a}_t$ to minimize $\E{ \sum_{t=1}^{T} c(\bm{a}_t)}$. This objective is typically reformulated as an equivalent \emph{regret} minimization problem, where the (expected) regret is defined as

\begin{equation}
    \text{Regret}(T) := \left(\sum_{t\in [T]} \E{c(\bm{a}_t)}\right) - T \cdot \min_{\bm{a} \in \mathcal{I}} \E{c(\bm{a})} . \label{eq:exact_regret}
\end{equation}

To connect this to the probabilistic bottleneck identification model introduced in the previous section, we let each edge $e \in E$ in the graph $G$ correspond to exactly one base arm $i \in \mathcal{A}$. For the online minimax path problem, the feasible set of super arms is then the set of all admissible paths in the graph, where the paths are directed or undirected depending on the type of graph. The feedback of each base arm $i$ is simply the Gaussian weight of the matching edge $e$, with known variance $\sigma_i^2$ and unknown mean $\theta^*_i$. 

We denote the expected cost of a super arm $f_{\bm{\theta}}(\bm{a})$, where $\bm{\theta}$ is a mean vector and $f_{\bm{\theta}}(\bm{a}) := \E{\max_{i \in \bm{a}} C_i}$ such that $C_i \sim \mathcal{N}(\theta_i, \sigma_i^2)$. For Bayesian bandit settings and algorithms, it is common to consider the notion of \emph{Bayesian regret}, with an additional expectation over problem instances drawn from the prior distribution (where we denote the prior distribution $\lambda$, over mean vectors $\bm{\theta}^*$):

\begin{equation}
    \text{BayesRegret}(T) := \mathbb{E}_{\bm{\theta}^* \sim \lambda} \left[ \E{ \left(\sum_{t \in [T]} f_{\bm{\theta}^*} (\bm{a}_t)\right) - T \cdot \min_{\bm{a} \in \mathcal{I}} f_{\bm{\theta}^*} (\bm{a}) \cond \bm{\theta}^*} \right] .
\end{equation}

\subsection{Thompson Sampling with exact objective}

It is not sufficient to find the super arm $\bm{a}$ which minimizes $f_{\bm{\mu}_t}(\bm{a})$ in each time step $t$, since a strategy which is \emph{greedy} with respect to possibly imperfect current cost estimates may converge to a sub-optimal super arm. Thompson Sampling is one of several methods developed to address the trade-off between exploration and exploitation in stochastic online learning problems. It has been shown to exhibit good performance in many formulations, e.g., linear contextual bandits and combinatorial semi-bandits.

The steps performed in each time step $t$ by Thompson Sampling, adapted to our setting, are described in Algorithm \ref{alg:exact_algorithm}. First, a mean vector $\Tilde{\bm{\theta}}$ is sampled from the current posterior distribution (or from the prior in the first time step). Then, an arm $\bm{a}_t$ is selected which minimizes the expected cost $f_{\Tilde{\bm{\theta}}} (\bm{a}_t)$ with respect to the sampled mean vector. These first two steps are equivalent to selecting the arm according to the posterior probability of it being optimal. In combinatorial semi-bandit problems, the method of finding the best super arm according to the sampled parameters is often called an \emph{oracle}.

When the super arm $\bm{a}_t$ is played, the environment reveals the feedback $X_{i,t}$ if and only if $i \in \bm{a}_t$, which is a property called \emph{semi-bandit feedback}. Finally, these observations are used to update the posterior distribution parameters.

\begin{algorithm}[ht]
\caption{TS for minimax paths (exact)}
\label{alg:exact_algorithm}
\textbf{Input}: Prior parameters $\bm{\mu}_0, \bm{\varsigma}_0$
\begin{algorithmic}[1] 
\State  For each base arm, play a super arm which contains it.
\For{$t \leftarrow 1, \dots, T$}
\For{$i \in \mathcal{A}$}
\State  $\Tilde{\theta}_i \leftarrow $ Sample from posterior $\mathcal{N}\left(\mu_{i,t-1}, \varsigma_{i,t-1}^2\right)$
\EndFor
\State  $\bm{a}_t \leftarrow \arg \min_{\bm{a} \in \mathcal{I}} f_{\Tilde{\bm{\theta}}} (\bm{a})$ \label{line:exact_alg_oracle}
\State  Play arm $\bm{a}_t$, observe feedback $X_{j,t}$ for $j \in \bm{a}_t$
\State  Compute $\bm{\mu}_{t}, \bm{\varsigma}_{t}$ with feedback using Eqs. \ref{eq:posterior_variance} and \ref{eq:posterior_mean}
\EndFor
\end{algorithmic}
\end{algorithm}

\subsection{Regret analysis of Thompson Sampling for minimax paths}

We use the technique to analyze the Bayesian regret of Thompson Sampling for general bandit problems introduced by \cite{russo2014learning} and further elaborated by \cite{slivkins2019}, carefully adapting it to our problem setting. This technique was originally devised to enable convenient conversion of existing UCB regret analyses to Thompson Sampling, but can also be applied to new TS applications. Here, we do a novel extension to combinatorial bandits with minimax super-arm cost functions, which includes establishing concentration properties for the mean estimates of the \emph{non-linear} super-arm costs. In the rest of this section, we outline the most important steps of the proof of Theorem \ref{alg:exact_algorithm}, leaving technical details to the supplementary material (Appendix \ref{sec:supplementary}). In the analysis, for convenience, we assume that $T \geq d$.

\begin{theorem}
\label{thm:regret_main}
The Bayesian regret of Algorithm \ref{alg:exact_algorithm} is $\mathcal{O}(d\sqrt{T \log T})$.
\end{theorem}

We initially define a sequence of upper and lower confidence bounds, for each time step $t$:

\begin{align*}
    &U_t(\bm{a}) := f_{\hat{\bm{\theta}}_{t-1}}(\bm{a})  + \max_{i \in \bm{a}} \sqrt{\frac{32 \log T}{N_{t-1}(i)}} \\
    &L_t(\bm{a}) := f_{\hat{\bm{\theta}}_{t-1}}(\bm{a})  - \max_{i \in \bm{a}} \sqrt{\frac{32 \log T}{N_{t-1}(i)}}
\end{align*}

where $\hat{\theta}_{i,t}$ is the average feedback of base arm $i \in \mathcal{A}$ until time $t$, $\hat{\bm{\theta}}_{t}$ is the average feedback vector for all arms in $\mathcal{A}$, and $N_t(i)$ is the number of times base arm $i \in \mathcal{A}$ has been played as part of a super arm until time $t$.

\begin{lemma}
\label{lem:regret_decomposition_main}
For Algorithm \ref{alg:exact_algorithm}, we have that:

\begin{align*}
    &\text{BayesRegret}(T) = \\
    &\sum_{t \in [T]} \E{U_t(\bm{a}_t) - L_t(\bm{a}_t)} + \sum_{t \in [T]} \E{f_{\bm{\theta}^*}(\bm{a}_t) - U_t (\bm{a}_t) } + \sum_{t \in [T]} \E{L_t (\bm{a}^*) - f_{\bm{\theta}^*}(\bm{a}^*) }.
\end{align*}
\end{lemma}

This Bayesian regret decomposition is a direct application of Proposition 1 of \cite{russo2014learning}. It utilizes the fact that given the history of selected arms and received feedback until time $t$, the played super arm $\bm{a}_t$ and the best possible super arm $\bm{a}^* := \arg \min_{\bm{a} \in \mathcal{I}} f_{\bm{\theta}^*}(\bm{a})$ are identically distributed under Thompson Sampling. Furthermore, also given the history, $U_t(\bm{a})$ and $L_t(\bm{a})$ are deterministic functions of the super arm $\bm{a}$. This enables the decomposition of the regret into terms of the expected confidence width, the expected overestimation of the super arm with least mean cost, and the expected underestimation of the selected super arm. By showing that $f_{\bm{\theta}^*}(\bm{a}) \in [L_t(\bm{a}), U_t(\bm{a})]$ with high probability, we can bound the last two of these terms.

\begin{lemma}
\label{lem:bad_event_regret_terms_main}
For any $t \in [T]$, we have that $\E{f_{\bm{\theta}^*}(\bm{a}_t) - U_t (\bm{a}_t) } \leq \frac{4d}{T}$ and $\E{L_t (\bm{a}^*) - f_{\bm{\theta}^*}(\bm{a}^*) } \leq \frac{4d}{T}$.
\end{lemma}

Both terms are bounded in the same way, for which we need a few intermediary results. Focusing on the underestimation of the played super arm, we can see that:

\begin{align*}
    &\E{f_{\bm{\theta}^*}(\bm{a}_t) - U_t (\bm{a}_t) } \\ 
    &= \E{f_{\bm{\theta}^*}(\bm{a}_t) - f_{\hat{\bm{\theta}}_{t-1}}(\bm{a}_t) - \max_{i \in \bm{a}_t} \sqrt{\frac{32 \log T}{N_{t-1}(i)}}} .
\end{align*}

First, in Lemma \ref{lem:reward_function_bound_main}, the difference between the true mean cost $f_{\bm{\theta}^*}(\bm{a})$ of a super arm $\bm{a}$ and the corresponding estimated mean $f_{\hat{\bm{\theta}}}(\bm{a})$ is bounded. The resulting upper bound is the maximum of the differences of the true and estimated means of each individual base arm feedback, such that:

\begin{lemma}
\label{lem:reward_function_bound_main}
For any super arm $\bm{a} \in \mathcal{I}$ and time step $t \in [T]$, we have that $\vert f_{\bm{\theta}^*}(\bm{a}) - f_{\hat{\bm{\theta}}_{t-1}}(\bm{a}) \vert \leq 2 \max_{i\in\bm{a}} \vert \theta^*_i - \hat{\theta}_{i,t-1} \vert$.
\end{lemma}

This is achieved by decomposing the absolute value into a sum of the positive and negative portions of the difference, then bounding each individually. Focusing on the positive portion by assuming that $f_{\bm{\theta}^*}(\bm{a}) \geq f_{\hat{\bm{\theta}}_{t-1}}(\bm{a})$, and letting $Z_i \sim \mathcal{N}(\hat{\theta}_{i,t-1}, \sigma^2_i)$, $Y_i \sim \mathcal{N}(\theta^*_{i}, \sigma^2_i)$, $\delta_{i,t-1} :=  \theta^*_i - \hat{\theta}_{i,t-1}$ and $Q_i := Y_i - \delta_{i,t-1}$, for $i \in \bm{a}$, we can see that:

\begin{align*}
    &f_{\bm{\theta}^*}(\bm{a}) - f_{\hat{\bm{\theta}}_{t-1}}(\bm{a})\\
    &= \E{\max_{i \in \bm{a}} Y_i} - \E{\max_{i \in \bm{a}} Z_i} \\
    &= \E{\max_{i \in \bm{a}} (Q_i + \delta_{i,t-1})} - \E{\max_{i \in \bm{a}} Z_i} \\
    &\leq \E{\max_{i \in \bm{a}} Q_i} + \max_{i \in \bm{a}} \delta_{i,t-1} - \E{\max_{i \in \bm{a}} Z_i} \\
    &= \max_{i \in \bm{a}} \delta_{i,t-1} \;\;.
\end{align*}

The negative portion is bounded in the same way, directly leading to the result of Lemma \ref{lem:reward_function_bound_main}. With this result, we can proceed with Lemma \ref{lem:bad_event_regret_terms_main}, where we let $[x]^+ := \max(0, x)$:

\begin{align}
    &\E{2\max_{i \in \bm{a}_t} \vert \delta_{i,t-1} \vert - \max_{i \in \bm{a}_t} \sqrt{\frac{32 \log T}{N_{t-1}(i)}}} \nonumber\\
    &\leq \E{2\max_{i \in \bm{a}_t} \left[\vert \delta_{i,t-1} \vert - \sqrt{\frac{8 \log T}{N_{t-1}(i)}}\right]^+} \nonumber\\
    &\leq 2 \sum_{i \in \mathcal{A}} \E{ \left[\vert \delta_{i,t-1} \vert - \sqrt{\frac{8 \log T}{N_{t-1}(i)}}\right]^+} \nonumber\\
    &= 2 \sum_{i \in \mathcal{A}} \Bigg( \prob{\vert \delta_{i,t-1} \vert > \sqrt{\frac{8 \log T}{N_{t-1}(i)}}} \cdot \label{eq:bad_event_probability}\\
    &\;\;\;\;\E{ \vert \delta_{i,t-1} \vert - \sqrt{\frac{8 \log T}{N_{t-1}(i)}} \;\;\bigg\vert\;\; \vert \delta_{i,t-1} \vert > \sqrt{\frac{8 \log T}{N_{t-1}(i)}}} \Bigg) \label{eq:bad_event_expectation}
\end{align}

The probability in Eq. \ref{eq:bad_event_probability} is of the event that the difference between the estimated and true means of an arm $i$ exceeds the confidence radius $\sqrt{8 \log T / N_{t-1}(i)}$, while Eq. \ref{eq:bad_event_expectation} is the expected difference conditional on that event. We bound Eq. \ref{eq:bad_event_probability} with Lemma \ref{lem:bad_event_probability_main} and Eq. \ref{eq:bad_event_expectation} with Lemma \ref{lem:bad_event_expectation_main}.

\begin{lemma}
\label{lem:bad_event_probability_main}
$\prob{\forall t \in [T]\; \forall i \in \mathcal{A},\;\; \vert \delta_{i,t-1} \vert \leq \sqrt{\frac{8 \log T}{N_{t-1}(i)}}} \geq 1 - \frac{2}{T}$.
\end{lemma}

It is now sufficient to show that the difference $\delta_{i,t-1}$ is small for all base arms $i \in \mathcal{A}$ with high probability, which we accomplish using a standard concentration analysis through application of Hoeffding's inequality and union bounds.

\begin{lemma}
\label{lem:bad_event_expectation_main}
For any $t \in [T]$ and $i \in \mathcal{A}$, we have

\begin{align*}
    \E{ \vert \delta_{i,t-1} \vert - \sqrt{\frac{8 \log T}{N_{t-1}(i)}} \;\;\bigg\vert\;\; \vert \delta_{i,t-1} \vert > \sqrt{\frac{8 \log T}{N_{t-1}(i)}}} \leq 1 .
\end{align*}
\end{lemma}

Though the rewards are unbounded, this expectation can be bounded by utilizing the fact that the mean of a truncated Gaussian distribution is increasing in the mean of the distribution before truncation, by Theorem 2 of \cite{horrace2015moments}. We can see that:

\begin{align*}
    &\E{ \vert \delta_{i,t-1} \vert - \sqrt{\frac{8 \log T}{N_{t-1}(i)}} \;\;\bigg\vert\;\; \vert \delta_{i,t-1} \vert > \sqrt{\frac{8 \log T}{N_{t-1}(i)}}} \\
    &= \E{ \delta_{i,t-1} - \sqrt{\frac{8 \log T}{N_{t-1}(i)}} \;\;\bigg\vert\;\;  \delta_{i,t-1} -  \sqrt{\frac{8 \log T}{N_{t-1}(i)}} > 0} \\
    &\leq \E{ \delta_{i,t-1} \;\;\bigg\vert\;\;  \delta_{i,t-1} > 0} .
\end{align*}

We know that $\delta_{i,t-1}$ is zero-mean Gaussian with variance at most one, hence $\E{ \delta_{i,t-1} \;\;\bigg\vert\;\;  \delta_{i,t-1} > 0} \leq 1$.

With the result from Lemma $\ref{lem:bad_event_regret_terms_main}$, the last two terms of the regret decomposition in Lemma \ref{lem:regret_decomposition_main} are bounded by constants in $T$. Focusing on the remaining term, we just need to show that $\sum_{t \in [T]} \E{U_t(\bm{a}_t) - L_t(\bm{a}_t)} \leq \mathcal{O} (d \sqrt{T \log T})$ to prove Theorem \ref{thm:regret_main}:

\begin{align*}
    &\sum_{t \in [T]} \E{U_t(\bm{a}_t) - L_t(\bm{a}_t)}\\
    &= \sqrt{128 \log T} \sum_{t\in [T]} \E{\max_{i\in \bm{a}_t} \frac{1}{\sqrt{N_{t-1}(i)}}} \\
    &\leq \sqrt{128 \log T} \sum_{t\in [T]} \E{\sum_{i\in \bm{a}_t} \frac{1}{\sqrt{N_{t-1}(i)}}} \\
    &= \sqrt{128 \log T} \sum_{i\in \mathcal{A}} \E{\sum_{t:i\in \bm{a}_t} \frac{1}{\sqrt{N_{t-1}(i)}}} \\
    &\leq \sqrt{128 \log T} \sum_{i\in \mathcal{A}} \E{2 \sqrt{N_{T}(i)}}  \\
    &\leq \sqrt{128 \log T} \cdot \E{2 \sqrt{d \sum_{i\in \mathcal{A}} N_{T}(i)}}  \\
    &\leq \sqrt{128 \log T} \cdot \E{2 \sqrt{d^2 T}} \\
    &= 2 d \sqrt{128  T \log T}  \\
    &= \mathcal{O} (d \sqrt{T \log T}) .
\end{align*}

We note that the final upper bound is tight up to a polylogarithmic factor, according to existing lower bounds for combinatorial semi-bandit problems \citep{kveton2015tight}.

\subsection{Thompson Sampling with approximate objective}

Unfortunately, exact expressions for computing the expected maximum of Gaussian random variables only exist when the variables are few. In other words, we cannot compute $f_{\bm{\theta}} (\bm{a})$ exactly for a super arm $\bm{a}$ containing many base arms, necessitating some form of approximation approach. While it is possible to approximate $f_{\bm{\theta}} (\bm{a})$ through e.g., Monte Carlo simulations, we want to be able to perform the cost minimization step using a computationally efficient oracle. 

We note that, even with the capability to exactly compute $f_{\bm{\theta}} (\bm{a})$, it would not be feasible to solve the minimization problem in line \ref{line:exact_alg_oracle} of Algorithm \ref{alg:exact_algorithm}. The expected cost $f_{\bm{\theta}} (\bm{a})$ of a super arm $\bm{a}$ (i.e., the expected maximum base arm feedback) depends not only on the individual expected values of the base arm feedback distributions, but also on the shape of the joint distribution of all base arms in $\bm{a}$. Due to this fact, the stochastic version of the minimization problem lacks the property of optimal substructure (i.e., an optimal path does not necessarily consist of optimal sub-paths). For the deterministic version of the problem, as defined in Eq. \ref{eq:MM}, the presence of this property enables the usage of computationally efficient dynamic programming strategies, like Dijkstra's algorithm, which is consequently infeasible with the objective in Algorithm \ref{alg:exact_algorithm}.

Therefore, we propose the approximation method outlined in Algorithm \ref{alg:approximation_algorithm}, where the minimization step of line \ref{line:approx_alg_oracle} has been modified from Algorithm \ref{alg:exact_algorithm} with an alternative super arm cost function $\Tilde{f}_{\Tilde{\bm{\theta}}} (\bm{a}) := \max_{i \in \bm{a}} \Tilde{\theta}_i$. Switching objectives, from finding the super arm which minimizes the expected maximum base arm feedback, to instead minimize the maximum expected feedback, has the benefit of allowing us to utilize the efficient deterministic minimax path algorithms introduced earlier for both directed and undirected graphs. For directed graphs, the modified version of Dijkstra's algorithm in \cite{berman1987optimal} has a worst-case running time of $\mathcal{O}(\vert E \vert + \vert V \vert \log \vert V \vert)$ with an efficient implementation using Fibonacci heaps \cite{fredman1987}. Similarly, for undirected graphs, finding an MST (and subsequently a minimax path) can be achieved using Prim's algorithm \cite{prim1957}, with the same running time of $\mathcal{O}(\vert E \vert + \vert V \vert \log \vert V \vert)$ if Fibonacci heaps are used. The other operations performed for each $t \in [T]$ in Algorithm \ref{alg:approximation_algorithm} (i.e., the posterior samples and updates) have a combined running time of, at worst, $\mathcal{O}(\vert E \vert)$. The same oracles are also used for the baseline algorithms evaluated in Sections \ref{sec:exp:road_networks} and \ref{sec:exp:social_network}, with comparable running times.

It is possible to use alternative notions of regret to evaluate combinatorial bandit algorithms with approximate oracles \citep{chen2013combinatorial, chen2016combinatorial}. For our experimental evaluation of Algorithm \ref{alg:approximation_algorithm}, we introduce the following definition of approximate regret:

\begin{equation*}
    \text{ApproxRegret(T)} := \left(\sum_{t \in [T]} \Tilde{f}_{\bm{\theta}^*} (\bm{a}_t)\right) - T \cdot \min_{\bm{a} \in \mathcal{I}} \Tilde{f}_{\bm{\theta}^*} (\bm{a}) .
\end{equation*}

An alternative Bayesian bandit algorithm which can be used with the alternative objective is BayesUCB \citep{kaufmann2012bayesian}, which we use as a baseline for our experiments. Like Thompson Sampling, BayesUCB has been adapted to combinatorial semi-bandit settings \citep{nuara2018combinatorial, ijcai2020-0284}. Whereas Thompson Sampling in Algorithm \ref{alg:approximation_algorithm} encourages exploration by applying the oracle to parameters sampled the posterior distribution, with BayesUCB, the oracle is instead applied to optimistic estimates based on the posterior distribution. In practice, this is accomplished for our cost minimization problem by using lower quantiles of the posterior distribution of each base arm. This principle of selecting plausibly optimal arms is called \emph{optimism in the face of uncertainty} and is the underlying idea of all bandit algorithms based on UCB. 

We note that while in BayesUCB, as outlined in Algorithm 1 of \citep{kaufmann2012bayesian}, the horizon is used to calculate UCB values, the authors of that work also explain that upper quantiles of order $1 - 1/t$ (calculated without the horizon) achieve good results in practice. For that reason, we use lower quantiles of order $1/t$ in the version of BayesUCB studied in this work, making it an \emph{anytime} algorithm, like Thompson Sampling.

\begin{algorithm}[ht]
\caption{TS for minimax paths (approximation)}
\label{alg:approximation_algorithm}
\textbf{Input}: Prior parameters $\bm{\mu}_0, \bm{\varsigma}_0$
\begin{algorithmic}[1] 
\State For each base arm, play a super arm which contains it.
\For{$t \leftarrow 1, \dots, T$}
\For{$i \in \mathcal{A}$}
\State $\Tilde{\theta}_i \leftarrow $ Sample from posterior $\mathcal{N}\left(\mu_{i,t-1}, \varsigma_{i,t-1}^2\right)$
\EndFor
\State $\bm{a_t} \leftarrow \arg \min_{\bm{a} \in \mathcal{I}} \max_{i \in \bm{a}} \Tilde{\theta}_i$ \label{line:approx_alg_oracle}
\State Play arm $\bm{a}_t$, observe feedback $X_{j,t}$ for $j \in \bm{a}_t$
\State Compute $\bm{\mu}_{t}, \bm{\varsigma}_{t}$ with feedback using Eq. \ref{eq:posterior_variance} and \ref{eq:posterior_mean}
\EndFor
\end{algorithmic}
\end{algorithm}

To connect the different objectives in Algorithm \ref{alg:exact_algorithm} and Algorithm \ref{alg:approximation_algorithm}, we note that by Jensen's inequality, $\Tilde{f}_{\Tilde{\bm{\theta}}} (\bm{a}) \leq f_{\Tilde{\bm{\theta}}} (\bm{a})$ and that the approximation objective consequently will underestimate super arm costs. However, we establish an upper bound on this difference through Theorem \ref{thm:optimal_arm_difference_main}.

\begin{theorem}
\label{thm:optimal_arm_difference_main}
Given the optimal super arm $\bm{a^*}$ for Algorithm \ref{alg:exact_algorithm} and the optimal super arm $\Tilde{\bm{a}}^*$ for Algorithm \ref{alg:approximation_algorithm}, we have that $ f_{\bm{\theta}^*}(\Tilde{\bm{a}}^*) - f_{\bm{\theta}^*}(\bm{a}^*) \leq \sqrt{2 \log d}$.
\end{theorem}

For any super arm $\bm{a} \in \mathcal{I}$, let $Y_i$ for $i \in \bm{a}$ be Gaussian random variables with $Y_i \sim \mathcal{N}(\theta^*_i, \sigma_i^2)$. Furthermore, let $W_i := Y_i - \theta^*_i$, such that $W_i \sim \mathcal{N}(0, \sigma^2_i)$. Then, the following holds:

\begin{align*}
    &\E{\max_{i \in \bm{a}} Y_i} \\
    &= \E{\max_{i \in \bm{a}} (W_i + \theta_i^*)} \\
    &\leq \E{\max_{i \in \bm{a}} W_i} + \max_{i \in \bm{a}} \E{Y_i} \\
    &\leq \sqrt{2 \log d} + \max_{i \in \bm{a}} \E{Y_i} \, ,
\end{align*}

where the last inequality is due to Lemma 9 in \cite{orabona2015optimal} and since $\sigma_i^2 \leq 1$ for all $i \in \bm{a}$. We also note that, by Jensen's inequality, we have $\max_{i \in \bm{a}} \E{Y_i} \leq \E{\max_{i \in \bm{a}} Y_i}$. Moreover, by definition we know that $\bm{a}^* = \arg \min_{\bm{a} \in \mathcal{I}} \E{\max_{i \in \bm{a}} Y_i}$ and $\Tilde{\bm{a}}^* = \arg \min_{\bm{a} \in \mathcal{I}} \max_{i \in \bm{a}} \E{ Y_i}$. Consequently, we have,

\begin{align*}
    \max_{i \in \Tilde{\bm{a}}^*} \E{Y_i} \leq \max_{i \in \bm{a}^*} \E{Y_i}  
    \leq \E{\max_{i \in \bm{a}^*} Y_i} \leq \E{\max_{i \in \Tilde{\bm{a}}^*} Y_i}  
    \leq \sqrt{2 \log d} + \max_{i\in \Tilde{\bm{a}}^*} \mathbb{E}\left[ Y_i \right] .
\end{align*}

Hence, we can conclude that

\begin{align*}
    f_{\bm{\theta}^*}(\Tilde{\bm{a}}^*) - f_{\bm{\theta}^*}(\bm{a}^*) = \E{\max_{i \in \Tilde{\bm{a}}^*} Y_i} - \E{\max_{i \in \bm{a}^*} Y_i}  
    \leq \sqrt{2 \log d}\, .
\end{align*}

In other words, Theorem \ref{thm:optimal_arm_difference_main} holds and the optimal solutions of the exact Algorithm \ref{alg:exact_algorithm} and the approximate Algorithm \ref{alg:approximation_algorithm} differ by at most $\sqrt{2 \log d}$. This bound is independent of the mean vector $\bm{\theta}^*$, depending only on the number of base arms and that the variance is bounded.

\section{Experimental Results}

In this section, we conduct bottleneck identification experiments using Algorithm \ref{alg:approximation_algorithm} for two real-world applications, i) road (transport) networks, and ii) collaboration (social) networks. These experiments are performed with an extended version of the simulation framework in \cite{tstutorial} and evaluated using our approximate definition of regret.
In addition, we compare Algorithm \ref{alg:exact_algorithm} to Algorithm \ref{alg:approximation_algorithm} through a toy example.

\subsection{Road networks}\label{sec:exp:road_networks}
A bottleneck in a network is a segment of a path in the network that obstructs or stops flow.
Identification of bottlenecks in a road network is a vital tool for traffic planners to analyze the network and prevent congestion.
In this application, our goal is to find the bottleneck between a source and a target, i.e., a road segment which is necessary to pass and also has minimal traffic flow. In the road network model, we let the nodes represent intersections and the directed edges represent road segments, with travel time divided by distance (seconds per meter) as edge weights.
The bottleneck between a pair of intersections is the minimum bottleneck over all paths connecting them, where the bottleneck for each of these paths is the largest weight over all road segments along it.
Note that in order for the bottleneck between a pair of intersections to have a meaning, there needs to exist at least one path connecting them.

We collect road networks of four cities, shown in Table \ref{tab:roadGraphs}, from \cite{OpenStreetMap}, where the average travel time as well as the distance is provided for each (directed) edge.
We simulate an environment with the stochastic edge weights sampled from $w_e \sim \mathcal{N}(\theta_e^*, \sigma_e^2)$, where the observation noise is $\sigma_e = 0.4$.
For the experiments, the environment samples the true unknown mean $\theta_e^*$ from the known prior $\theta_e^* \sim \mathcal{N}(\mu_{e,0}, {\varsigma_{e,0}}^2)$, where
$\varsigma_{e,0} = 0.4 \textrm{s/m}$, and $\mu_{e,0}$
is the average travel time divided by distance provided by OpenStreetMap~(OSM).

\begin{table}[t]
\begin{center}
\caption{A description of the road networks. } \label{tab:roadGraphs}
\begin{tabular}{@{}lcccc@{}}
\toprule
& \multicolumn{4}{@{}c@{}}{Road network} \\ \cmidrule{2-5}
& Eindhoven & Manhattan & Oslo & Salzburg \\ \midrule
\#Node & 7501 & 4593 & 8153 & 2921 \\
\#Edge & 10776 & 8130 & 11192 & 3848 \\
Avg. Degree & 2.873 & 3.540 & 2.745 & 2.634 \\
\bottomrule
\end{tabular}
\end{center}
\end{table}

We consider one greedy agent (GR) and two $\epsilon_t$-greedy agents (e-GR) as baselines. The greedy agent (GR) always chooses the path with the lowest current estimate of expected cost. In each time step, each e-GR agent, with probability $\epsilon_t$ decreasing with $t$ (specifically, we let $\epsilon_t = \min(1, 1/\sqrt{t})$), chooses a random path, and acts like the greedy agent otherwise. In our experiments, we implement the two e-GR agents based on the combinatorial version of $\epsilon_t$-greedy introduced in Algorithm 1 in the Supplementary Material of \cite{chen2013combinatorial}. The first e-GR agent chooses a path between the source and the target containing a uniformly chosen random node (e-GR-N), and the second e-GR agent chooses a path with a uniformly selected random edge (e-GR-E). We evaluate how the performance of the Thompson Sampling agent (TS) and the BayesUCB agent (B-UCB) compare to the baselines.
We run the simulations with all five agents for each road network and report the cumulative regret at a given horizon $T$, averaged over five repetitions. The horizon is chosen such that the instant regret is almost stabilized for the agents.

\begin{table}[b]
\begin{center}
\begin{minipage}{\textwidth}
\caption{Average cumulative regret and corresponding standard error (SE) over five runs, at the horizon $T = 6000$, for Thompson  Sampling  (TS),  BayesUCB  (B-UCB), $\epsilon_t$-greedy agents (e-GR-N and e-GR-E), and Greedy (GR) agent. } 
\label{tab:inst} 
\begin{tabular*}{\textwidth}{@{\extracolsep{\fill}}lcccc@{\extracolsep{\fill}}}
\toprule%
  & \multicolumn{4}{@{}c@{}}{Road network} \\ \cmidrule{2-5}
  & Eindhoven & Manhattan & Oslo  & Salzburg \\ \midrule
TS      & 271.8 $\pm$ 46.3  & 449.0  $\pm$ 53.0   & 226.7 $\pm$ 69.1  & 131.7  $\pm$ 13.4 \\ 
B-UCB   & 483.8 $\pm$ 112.9 & 670.5  $\pm$ 68.9   & 339.7 $\pm$ 88.3  & 259.2  $\pm$ 41.4 \\
e-GR-N  & 838.8 $\pm$ 191.2 & 1232.9 $\pm$ 72.1   & 379.3 $\pm$ 115.9 & 653.8  $\pm$ 175.7\\
e-GR-E  & 928.7 $\pm$ 191.7 & 1120.0 $\pm$ 113.8  & 405.7 $\pm$ 113.1 & 609.4  $\pm$ 87.0 \\
GR      & 936.9 $\pm$ 223.2 & 1116.7 $\pm$ 142.9  & 511.2 $\pm$ 140.1 & 1159.0 $\pm$ 155.2\\
\bottomrule
\end{tabular*}
\end{minipage}
\end{center}
\end{table}

Table \ref{tab:inst} shows the average cumulative regrets and their corresponding standard error over five runs at the horizon $T$. For all four road networks, the TS agent incurs the lowest average cumulative regret and standard error over five runs. Then, B-UCB follows TS and yields a better result than the baselines (GR and both e-GR variants).

\begin{figure*}[p] \centering
    \begin{subfigure}{0.40\textwidth}
    \centering
    \captionsetup{justification=centering}
     \includegraphics[width=\textwidth]{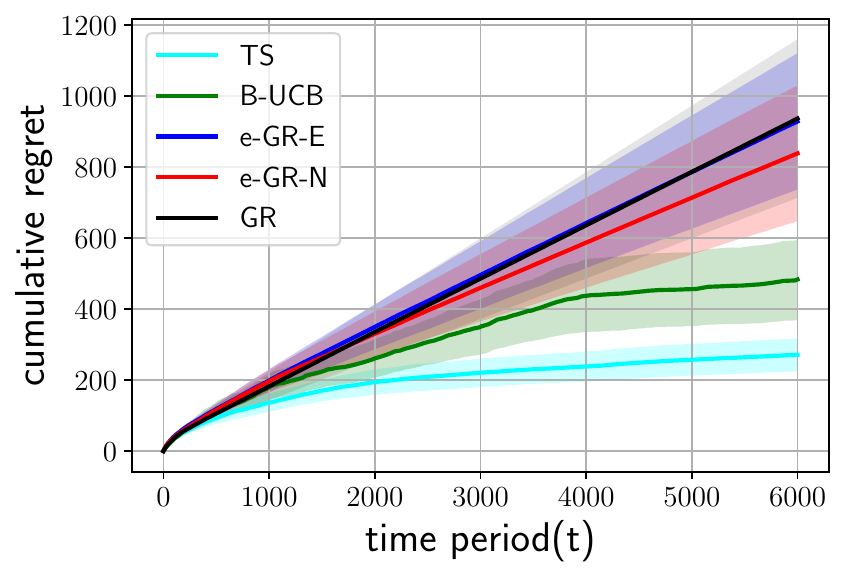}
        \caption{\footnotesize{}}
        \label{fig:EindhovenCul}
    \end{subfigure}
    \begin{subfigure}{0.30\textwidth}
    \centering
    \captionsetup{justification=centering}
        \includegraphics[width=\textwidth]{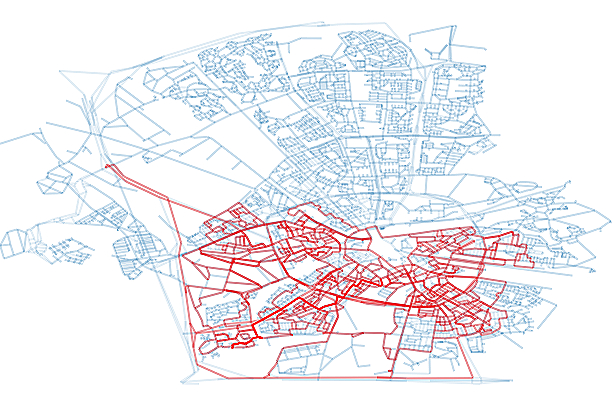}
         \caption{ \footnotesize{} }
         \label{fig:EindhovenRoad}
    \end{subfigure}
    \begin{subfigure}{0.40\textwidth}
    \centering
    \captionsetup{justification=centering}
      \includegraphics[width=\textwidth]{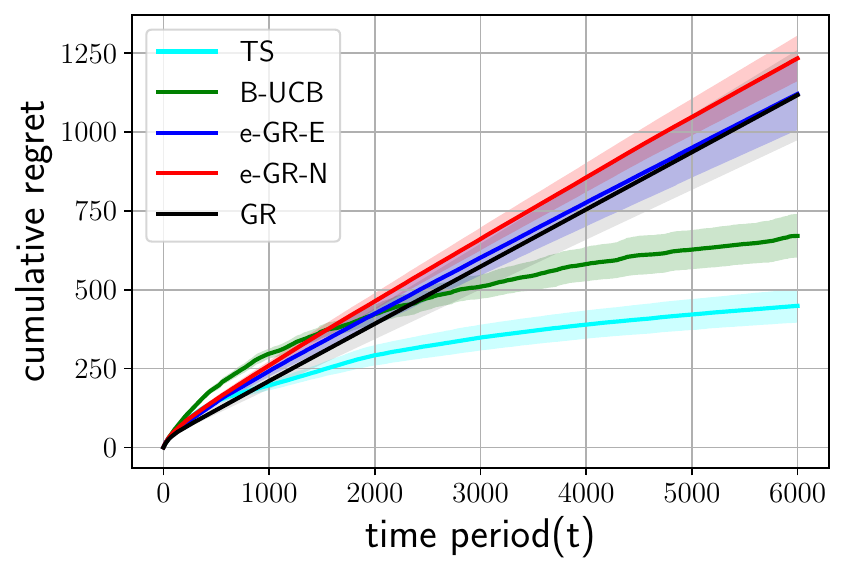}
         \caption{\footnotesize{}}
         \label{fig:ManhattanCul}
    \end{subfigure} 
    \begin{subfigure}{0.30\textwidth}
    \centering
    \captionsetup{justification=centering}
        \includegraphics[width=0.5\textwidth]{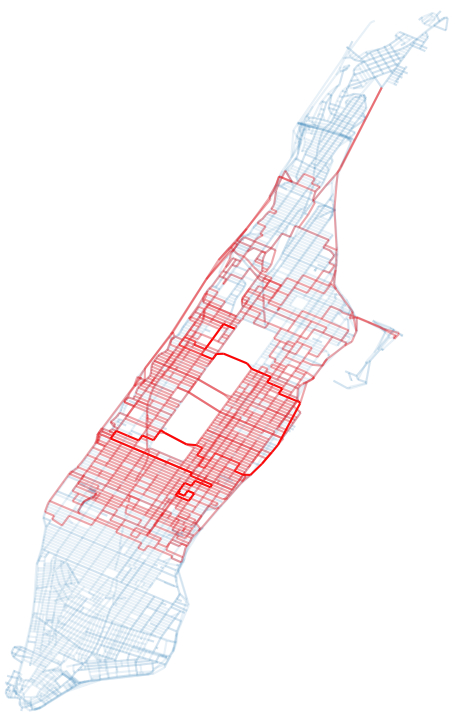}
         \caption{ \footnotesize{} }
         \label{fig:ManhattanRoad}
    \end{subfigure}
    \begin{subfigure}{0.40\textwidth}
    \centering
    \captionsetup{justification=centering}
      \includegraphics[width=\textwidth]{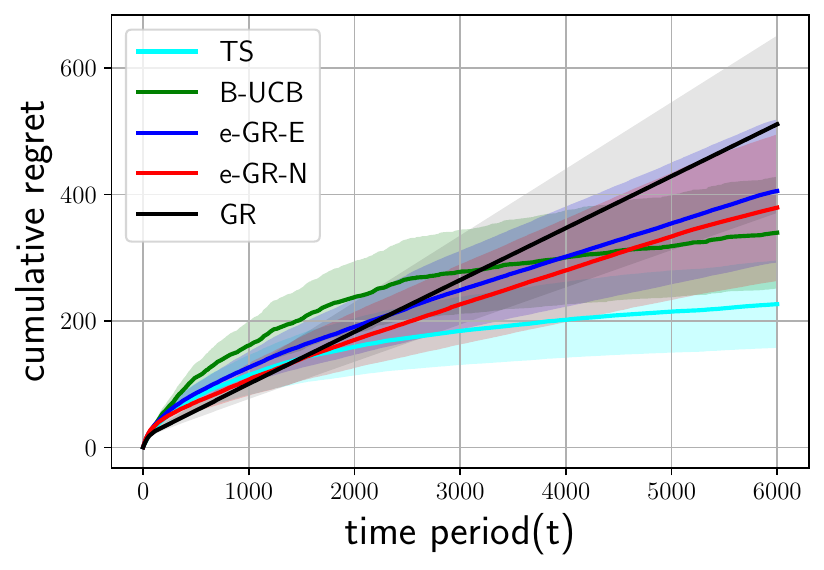}
         \caption{\footnotesize{}}
         \label{fig:OsloCul}
    \end{subfigure}
    \begin{subfigure}{0.30\textwidth}
    \centering
    \captionsetup{justification=centering}
        \includegraphics[width=\textwidth]{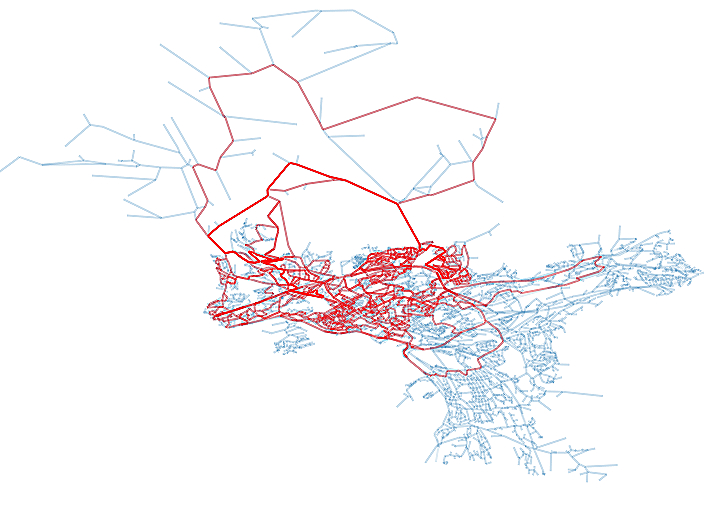}
         \caption{ \footnotesize{} }
         \label{fig:OsloRoad}
    \end{subfigure}
    \begin{subfigure}{0.40\textwidth}
    \centering
    \captionsetup{justification=centering}
      \includegraphics[width=\textwidth]{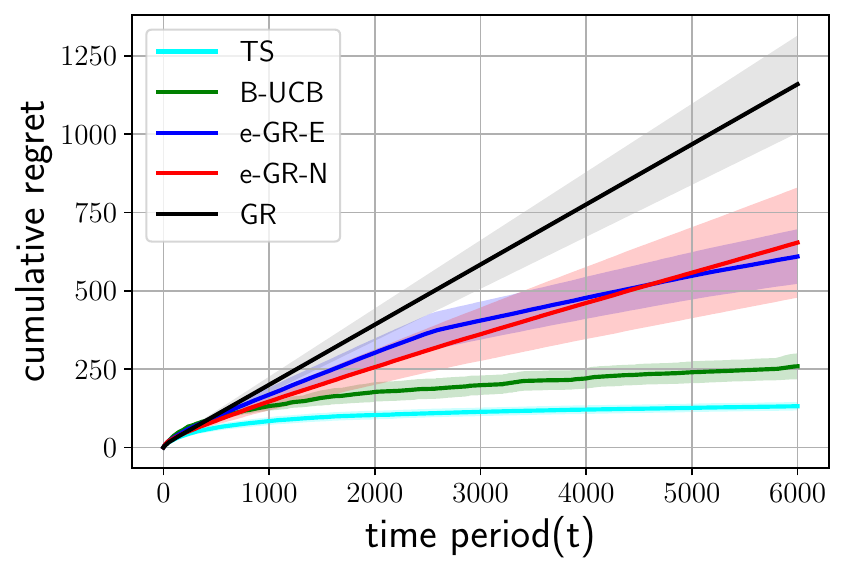}
         \caption{\footnotesize{}}
         \label{fig:SalzburgCul}
    \end{subfigure}
    \begin{subfigure}{0.30\textwidth}
    \centering
    \captionsetup{justification=centering}
        \includegraphics[width=\textwidth]{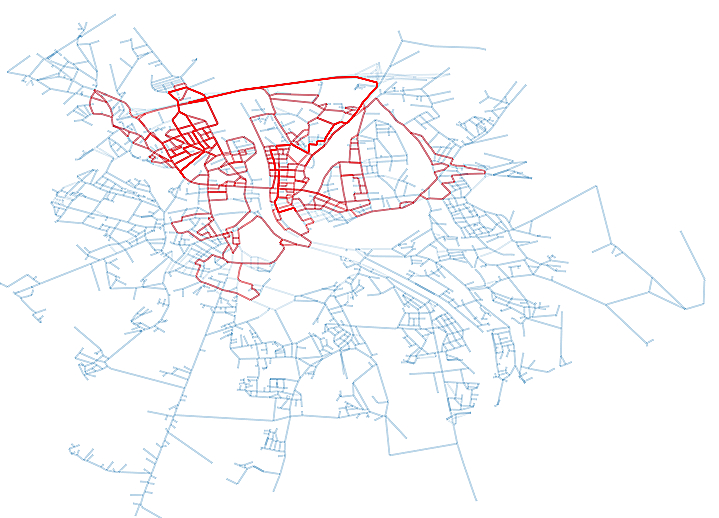}
         \caption{ \footnotesize{} }
         \label{fig:SalzburgRoad}
    \end{subfigure}    
    \caption{Cumulative regret averaged over 5 runs with shaded standard error bars, for Thompson Sampling  (TS), Bayes UCB (B-UCB), $\epsilon_t$-greedy agents (e-GR-N and e-GR-E), and greedy (GR) with horizon $T= 6000$, on  Eindhoven (a), Manhattan (c), Oslo (e) and Salzburg (g) road networks. Visualizations of the paths explored by the TS agent are shown in red, for  Eindhoven (b), Manhattan (d), Oslo (f) and Salzburg (h) road networks. Opacity illustrates the exploration of each of the road segments.
    }\label{fig:results}
\end{figure*}

Figure \ref{fig:results} illustrates the average cumulative regret with standard error (SE) bars on the road networks of the four aforementioned cities. 
For Eindhoven, Figure \ref{fig:EindhovenCul} shows the average cumulative regret, where at horizon $T=6000$ the TS agent yields the lowest cumulative regret. Then, B-UCB follows TS and achieves a better result compared to the other baselines. 
As time progresses, we can see that first TS and then B-UCB start saturating by performing sufficient exploration. With respect to the SE bars, there are differences between the five agents. The TS agent has the smallest SE bars. Figure \ref{fig:EindhovenRoad} visualizes the Eindhoven road network, where the paths explored by the TS agent are shown in red. The road segments explored (tried) more often by the TS agent are displayed more opaque. 
Figure \ref{fig:ManhattanCul}, \ref{fig:OsloCul}, and \ref{fig:SalzburgCul} show the average cumulative regret with SE bars for Manhattan, Oslo, and Salzburg, respectively. The results show that TS incurs the lowest cumulative regret and smallest SE bars. Then, B-UCB follows TS in both aspects and obtains a better result than the other baselines.

\begin{figure}[ht] \centering
  \includegraphics[width=0.4\linewidth]{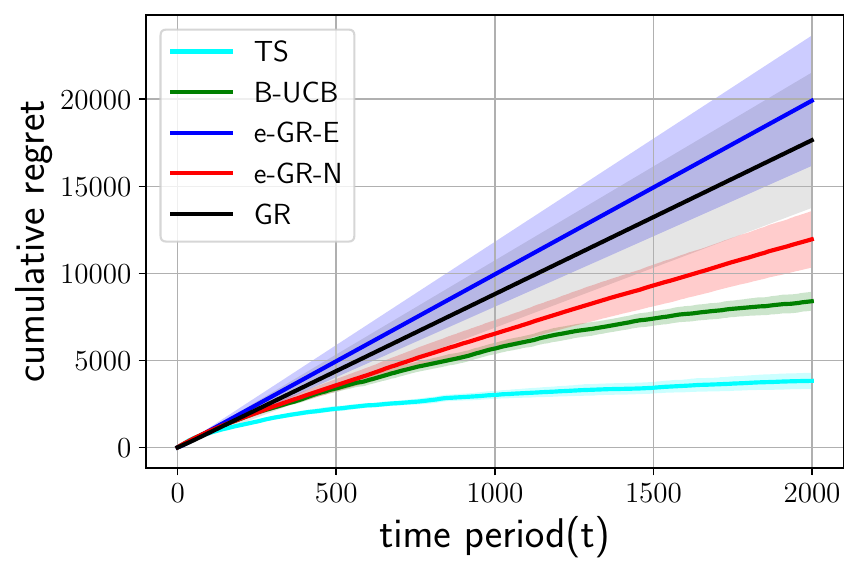}
    \caption{Cumulative regret averaged over 5 runs with shaded standard error bars, for Thompson Sampling  (TS), Bayes UCB (B-UCB), $\epsilon_t$-greedy agnets (e-GR-N and e-GR-E), and greedy (GR) with horizon $T= 2000$ for the collaboration network.  }\label{fig:geomCumulative}
\end{figure}

\subsection{Collaboration network}\label{sec:exp:social_network}

We consider a collaboration network from computational geometry (Geom) \citep{jones2002computational} as an application of our approach to social networks.
More specifically, we use the version provided by \cite{handcock2003statnet} and distributed among the Pajek datasets \citep{pajek} where certain author duplicates, occurring in minor or major name variations, have been merged. The \cite{handcock2003statnet} version is based on the BibTeX bibliography \citep{beebe2002}, to which the database from \cite{jones2002computational} has been exported.
The network has 9072 vertices representing the authors and 22577 edges with the edge weights representing the number of mutual works between a pair of authors.

We simulate an environment where each edge weight is sampled as $w_e \sim \mathcal{N}(\theta_e^*, \sigma_e^2)$, within which $\theta_e^*$ is regarded as the true (negative) mean number of shared publications between a pair of authors linked by the edge $e$, and the observation noise is $\sigma_e = 5$.
Furthermore, in this experiment, while the true negative mean number of mutual publications are assumed (by the agent) to be distributed according to the prior $\theta_e^* \sim \mathcal{N}(\mu_{e,0}, \varsigma_{e,0}^2)$ with $\varsigma_{e,0} = 10$, we instead generate the mean from a wider prior $\theta_e^* \sim \mathcal{N}(\mu_{e,0}, 20^2)$, simulating a scenario where the prior belief of the agent is too high.
The assumed mean $\mu_{e,0}$ of the prior is however consistent with the distribution from which $\theta_e^*$ is sampled, and is directly determined by the pairwise negative number of mutual collaborations from the dataset in \cite{handcock2003statnet}.

Figure \ref{fig:geomCumulative} shows the cumulative regret, averaged over five runs for the different agents with horizon $T=2000$, again chosen such that the regret is stabilized for all agents. One can see that the TS agent reaches the lowest cumulative regret, similar to the experimental studies on road networks. 

\subsection{Exact objective toy example}

\begin{figure}[t!]
\centering
\includegraphics[width=0.4\textwidth]{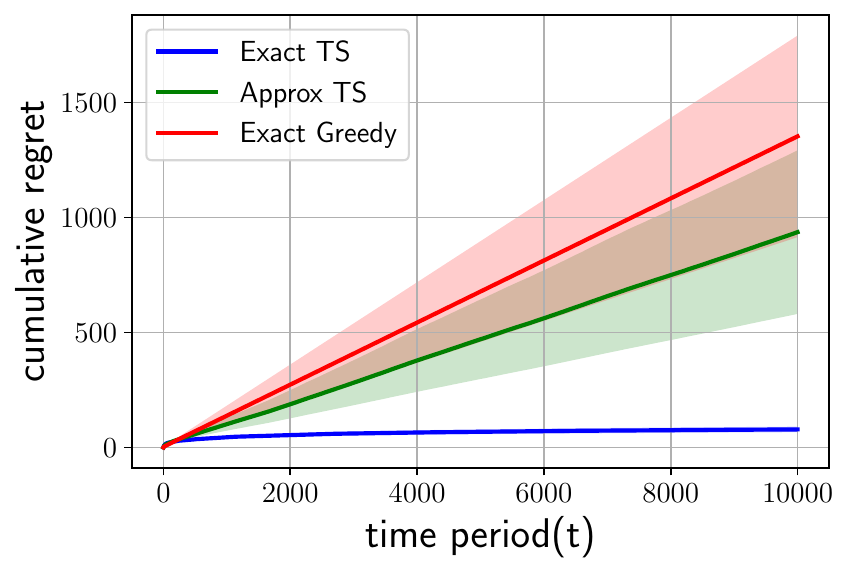}
\caption{Experimental results with average cumulative regret on the toy example, with $T=10000$, on exact TS, approximate TS and exact greedy.}
\label{fig:toy_example}
\end{figure}

While it is not feasible to evaluate Algorithm \ref{alg:exact_algorithm} on graphs representing real-life transportation or social networks, it is possible for small synthetic graphs. We construct a graph consisting of 6 nodes and 10 edges, with the source and target nodes connected by four paths of length 2 and four paths of length 3. For each edge $e$, we use the sample the mean from a standard Gaussian prior, such that $\theta_e^* \sim \mathcal{N}(0,1)$. The stochastic weights are then generated in each time step $t$ such that $w_{e,t} \sim \mathcal{N}(\theta^*_e, 1)$.

In order to calculate the expected cost of each path, we use existing exact expressions for the expected maximum of two \cite{clark1961greatest} and three \cite{lo2020improving, guptagaussians} independent Gaussian random variables. Instead of using an oracle, we simply enumerate the paths to find the one with minimum expected cost. 

In Figure \ref{fig:toy_example}, we compare Algorithm \ref{alg:exact_algorithm} (TS with exact objective) and Algorithm \ref{alg:approximation_algorithm} (TS with approximate objective) using the exact notion of (cumulative) regret as defined in Eq. \ref{eq:exact_regret}. Furthermore, we include a greedy baseline which also uses the exact objective. We use a horizon of $T=10000$ and average the results over 20 experiments, wherein each algorithm is applied to a problem instance sampled from the prior.

We can see that the regret of exact TS quickly saturates, while approximate TS and the greedy method tend to end up in sub-optimal solutions. For approximate TS, this is to be expected since optimal arms for the exact and approximate problems may be different. It is worth noting, however, that approximate TS performs better than the exact greedy method on average.

\section{Conclusion}
We developed an online learning framework for bottleneck identification in networks via minimax paths.
In particular, we modeled this task as a combinatorial semi-bandit problem for which we proposed a combinatorial version of Thompson Sampling. We then established an upper bound on the  Bayesian regret of the Thompson Sampling method.
To deal with the computational intractability of the problem, we devised an alternative problem formulation which approximates the original objective.
Finally, we investigated the framework on several directed and undirected real-world networks from  transport and collaboration domains. Our experimental results demonstrate its effectiveness  compared to alternatives such as greedy and B-UCB methods. 

\section*{Author Contributions}
All authors have contributed significantly to the conception, design and writing of this work. The theoretical analysis was primarily performed by Niklas {\AA}kerblom, and the experimental study was primarily performed by Fazeleh Sadat Hoseini.

\section*{Acknowledgments}
This work is partially funded by the Strategic Vehicle Research and Innovation Programme (FFI) of Sweden, through the project EENE (reference number: 2018-01937). We want to thank the reviewers for their helpful comments and suggestions. We also want to thank Emilio Jorge, Emil Carlsson and Tobias Johansson for insightful discussions around the proofs.

\appendix
\section{Technical Details of Regret Analysis}\label{sec:supplementary}

Here, we include detailed proofs for the theorems and lemmas in the main paper. We use the technique to analyze the Bayesian regret of Thompson Sampling for general bandit problems outlined by \cite{russo2014learning} and further detailed by \cite{slivkins2019}, carefully adapting it to our problem setting.

\begin{theorem}
The Bayesian regret of Algorithm \ref{alg:exact_algorithm} is $\mathcal{O}(d\sqrt{T \log T})$.
\end{theorem}

\begin{proof}
By Lemma \ref{lem:regret_decomposition} combined with Lemma \ref{lem:over_under_estimation}, we have

\begin{align*}
    \text{BayesRegret}(T) \leq 8d + \sum_{t \in [T]} \E{U_t(\bm{a}_t) - L_t(\bm{a}_t)} =& \\
     8 d + 2 \sum_{t \in [T]} \E{\max_{i\in \bm{a}_t} \sqrt{\frac{32 \log T}{N_{t-1}(i)}}} =& \\
     8 d + \sqrt{128 \log T} \sum_{t \in [T]} \E{\max_{i\in \bm{a}_t} \frac{1}{\sqrt{N_{t-1}(i)}}} \leq& \\
     8 d + \sqrt{128 \log T} \sum_{t \in [T]} \E{\sum_{i\in \bm{a}_t} \frac{1}{\sqrt{N_{t-1}(i)}}} =& \\
     8 d + \sqrt{128 \log T} \sum_{i\in \mathcal{A}} \E{\sum_{t:i\in \bm{a}_t} \frac{1}{\sqrt{N_{t-1}(i)}}} =& \\
     8 d + \sqrt{128 \log T} \sum_{i\in \mathcal{A}} \E{\sum_{j=1}^{N_{T}(i)} \frac{1}{\sqrt{j}}} \leq& \\
     \intertext{(See proof of Lemma 1 in \cite{russo2014learning})}
     8 d + 2 \sqrt{128 \log T} \sum_{i\in \mathcal{A}} \E{\sqrt{N_{T}(i)}} \leq& \\
     \intertext{(Cauchy-Schwarz inequality)}
     8 d + 2 \sqrt{128 \log T} \cdot \E{\sqrt{d \sum_{i\in \mathcal{A}} N_{T}(i)}} \leq& \\
     8 d + 2 \sqrt{128 \log T} \cdot \E{\sqrt{d^2 T}} =& \\
     8 d + 2 d \sqrt{128  T \log T} =& \\
     \mathcal{O}(d\sqrt{T \log T}) &
\end{align*}
\end{proof}

\begin{lemma}
\label{lem:regret_decomposition}
For Algorithm \ref{alg:exact_algorithm}, we have that $\text{BayesRegret}(T) = \sum_{t \in [T]} \E{U_t(\bm{a}_t) - L_t(\bm{a}_t)} + \\ \sum_{t \in [T]} \E{f_{\bm{\theta}^*}(\bm{a}_t) - U_t (\bm{a}_t) } + \sum_{t \in [T]} \E{L_t (\bm{a}^*) - f_{\bm{\theta}^*}(\bm{a}^*) }$.
\end{lemma}

\begin{proof}
By Proposition 1 in \cite{russo2014learning}, we can decompose the Bayesian regret of the algorithm in the following way:

\begin{align*}
    \text{BayesRegret}(T) = \sum_{t \in [T]} \E{f_{\bm{\theta}^*}(\bm{a}_t) - L_t (\bm{a}_t) } + \sum_{t \in [T]} \E{L_t (\bm{a}^*) - f_{\bm{\theta}^*}(\bm{a}^*) } =& \\
    \sum_{t \in [T]} \E{U_t(\bm{a_t}) - L_t(\bm{a_t})} + \sum_{t \in [T]} \E{f_{\bm{\theta}^*}(\bm{a}_t) - U_t (\bm{a}_t) } + \sum_{t \in [T]} \E{L_t (\bm{a}^*) - f_{\bm{\theta}^*}(\bm{a}^*) } &
\end{align*}
\end{proof}

\begin{lemma}
\label{lem:over_under_estimation}
For any $t \in [T]$, we have that $\E{f_{\bm{\theta}^*}(\bm{a}_t) - U_t (\bm{a}_t) } \leq \frac{4d}{T}$ and $\E{L_t (\bm{a}^*) - f_{\bm{\theta}^*}(\bm{a}^*) } \leq \frac{4d}{T}$.
\end{lemma}

\begin{proof}
\begin{align*}
    \E{f_{\bm{\theta}^*}(\bm{a}_t) - U_t (\bm{a}_t) } =& \\
    \E{f_{\bm{\theta}^*}(\bm{a}_t) - f_{\hat{\bm{\theta}}_{t-1}}(\bm{a}_t) - \max_{i \in \bm{a}_t} \sqrt{\frac{32 \log T}{N_{t-1}(i)}}} \leq& \\
    \intertext{(By Lemma \ref{lem:reward_function_bound})}
    \E{2\max_{i \in \bm{a}_t} \vert \theta^*_i - \hat{\theta}_{i,t-1} \vert - \max_{i \in \bm{a}_t} \sqrt{\frac{32 \log T}{N_{t-1}(i)}}} =& \\
    \intertext{(Let $j = \arg \max_{i \in \bm{a}_t} \vert \theta^*_i - \hat{\theta}_{i,t-1} \vert $)}
    \E{2 \vert \theta^*_j - \hat{\theta}_{j,t-1} \vert - \max_{i \in \bm{a}_t} \sqrt{\frac{32 \log T}{N_{t-1}(i)}}} \leq& \\
    \E{2 \vert \theta^*_j - \hat{\theta}_{j,t-1} \vert - \sqrt{\frac{32 \log T}{N_{t-1}(j)}}} \leq& \\
    \E{2 \left[\vert \theta^*_j - \hat{\theta}_{j,t-1} \vert - \sqrt{\frac{8 \log T}{N_{t-1}(j)}}\right]^+} \leq& \\
    \E{2 \sum_{i \in \mathcal{A}} \left[\vert \theta^*_i - \hat{\theta}_{i,t-1} \vert - \sqrt{\frac{8 \log T}{N_{t-1}(i)}}\right]^+} =& \\
    2 \sum_{i \in \mathcal{A}} \E{ \left[\vert \theta^*_i - \hat{\theta}_{i,t-1} \vert - \sqrt{\frac{8 \log T}{N_{t-1}(i)}}\right]^+} =& \\
    2 \sum_{i \in \mathcal{A}} \E{ \vert \theta^*_i - \hat{\theta}_{i,t-1} \vert - \sqrt{\frac{8 \log T}{N_{t-1}(i)}} \;\;\;\bigg\vert\;\;\; \vert \theta^*_i - \hat{\theta}_{i,t-1} \vert > \sqrt{\frac{8 \log T}{N_{t-1}(i)}}} \prob{\vert \theta^*_i - \hat{\theta}_{i,t-1} \vert > \sqrt{\frac{8 \log T}{N_{t-1}(i)}}} \leq& \\
    \intertext{(By Lemma \ref{lem:bad_event_expectation})}
    2 \sum_{i \in \mathcal{A}} \prob{\vert \theta^*_i - \hat{\theta}_{i,t-1} \vert > \sqrt{\frac{8 \log T}{N_{t-1}(i)}}} \leq& \\
    \intertext{(By Lemma \ref{lem:concentration})}
    2 \sum_{i \in \mathcal{A}} \frac{2}{T} \leq& \\
    \frac{4d}{T} &
\end{align*}

The proof for $\E{L_t (\bm{a}^*) - f_{\bm{\theta}^*}(\bm{a}^*) } \leq \frac{4d}{T}$ is done in the same way.
\end{proof}

\begin{lemma}
\label{lem:reward_function_bound}
For any super arm $\bm{a} \in \mathcal{I}$ and time step $t \in [T]$, we have that $\vert f_{\bm{\theta}^*}(\bm{a}) - f_{\hat{\bm{\theta}}_{t-1}}(\bm{a}) \vert \leq 2 \max_{i\in\bm{a}} \vert \theta^*_i - \hat{\theta}_{i,t-1} \vert$.
\end{lemma}

\begin{proof}

Let $Y_i$ and $Z_i$ for $i \in \bm{a}$ be Gaussian random variables with $Y_i \sim \mathcal{N}(\theta_i^*, \sigma^2_i)$ and $Z_i \sim \mathcal{N}(\hat{\theta}_{i,t-1}, \sigma^2_i)$.

\begin{align*}
\vert f_{\bm{\theta}^*}(\bm{a}) - f_{\hat{\bm{\theta}}_{t-1}}(\bm{a}) \vert =& \\
\left\vert \E{\max_{i \in \bm{a}} Y_i} - \E{\max_{i \in \bm{a}} Z_i} \right\vert =& \\
\intertext{(Define $[x]^+ := \max(0, x)$)}
\left[ \E{\max_{i \in \bm{a}} Y_i} - \E{\max_{i \in \bm{a}} Z_i} \right]^+ + \left[\E{\max_{i \in \bm{a}} Z_i} - \E{\max_{i \in \bm{a}} Y_i} \right]^+ =&
\intertext{(Let $\delta_{i,t-1} := \theta^*_i - \hat{\theta}_{i,t-1}$, $Q_i := Y_i - \delta_{i,t-1}$, $O_i := Z_i + \delta_{i,t-1}$)}
\left[ \E{\max_{i \in \bm{a}} (Q_i + \delta_{i,t-1})} - \E{\max_{i \in \bm{a}} Z_i} \right]^+ + \left[\E{\max_{i \in \bm{a}} (O_i - \delta_{i,t-1})} - \E{\max_{i \in \bm{a}} Y_i} \right]^+ \leq& \\
\left[ \E{\max_{i \in \bm{a}} Q_i + \max_{i \in \bm{a}} \delta_{i,t-1}} - \E{\max_{i \in \bm{a}} Z_i} \right]^+ + \left[\E{\max_{i \in \bm{a}} O_i + \max_{i \in \bm{a}} (-\delta_{i,t-1})} - \E{\max_{i \in \bm{a}} Y_i} \right]^+ =& \\
\left[ \E{\max_{i \in \bm{a}} Q_i} + \max_{i \in \bm{a}} \delta_{i,t-1} - \E{\max_{i \in \bm{a}} Z_i} \right]^+ + \left[\E{\max_{i \in \bm{a}} O_i} + \max_{i \in \bm{a}} (-\delta_{i,t-1}) - \E{\max_{i \in \bm{a}} Y_i} \right]^+ =& \\
\intertext{(Since $\E{\max_{i \in \bm{a}} Q_i} = \E{\max_{i \in \bm{a}} Z_i}$ and $\E{\max_{i \in \bm{a}} O_i} = \E{\max_{i \in \bm{a}} Y_i}$)}
\left[\max_{i \in \bm{a}} \delta_{i,t-1}\right]^+ + \left[\max_{i \in \bm{a}} (-\delta_{i,t-1})\right]^+ \leq& \\
2 \max_{i \in \bm{a}} \vert \delta_{i,t-1} \vert =& \\
2 \max_{i \in \bm{a}} \vert \theta^*_i - \hat{\theta}_{i,t-1} \vert &
\end{align*}
\end{proof}

\begin{lemma}
\label{lem:concentration}
$\prob{\forall t \in [T]\; \forall i \in \mathcal{A},\;\; \vert \theta^*_i - \hat{\theta}_{i,t-1} \vert \leq \sqrt{\frac{8 \log T}{N_{t-1}(i)}}} \geq 1 - \frac{2}{T}$.
\end{lemma}

\begin{proof}

We define $\Bar{v}_{i,m}$ as the average feedback of base arm $i$ for the first $m$ times it has been played as part of a super arm, i.e., such that $\hat{\theta}_{i,t} = \Bar{v}_{i,N_{t}(i)}$. Then:

\begin{align*}
    \prob{\forall t \in [T]\; \forall i \in \mathcal{A},\;\; \vert \theta^*_i - \hat{\theta}_{i,t-1} \vert \leq \sqrt{\frac{8 \log T}{N_{t-1}(i)}}} =& \\
    1 - \prob{\exists t \in [T]\; \exists i \in \mathcal{A},\;\; \vert \theta^*_i - \hat{\theta}_{i,t-1} \vert > \sqrt{\frac{8 \log T}{N_{t-1}(i)}}} \geq& \\
    \intertext{(Union bound)}
    1 - \sum_{i \in \mathcal{A}} \sum_{t \in [T]} \prob{\vert \theta^*_i - \hat{\theta}_{i,t-1} \vert > \sqrt{\frac{8 \log T}{N_{t-1}(i)}}} =& \\
    1 - \sum_{i \in \mathcal{A}} \sum_{t \in [T]} \sum_{m = 1}^{t-1} \prob{\vert \theta^*_i - \Bar{v}_{i,m} \vert > \sqrt{\frac{8 \log T}{m}} \;\land\; N_{t-1}(i) = m} \geq& \\
    1 - \sum_{i \in \mathcal{A}} \sum_{t \in [T]} \sum_{m = 1}^{t-1} \prob{\vert \theta^*_i - \Bar{v}_{i,m} \vert > \sqrt{\frac{8 \log T}{m}}} \geq& \\
    \intertext{(Hoeffding's inequality, specifically version in Theorem A.2 of \cite{slivkins2019})}
    1 - \sum_{i \in \mathcal{A}} \sum_{t \in [T]} \sum_{m = 1}^{t-1} \frac{2}{T^4} \geq& \\
    \intertext{(We assume that $T \geq d$)}
    1 - \frac{2}{T} &
\end{align*}
\end{proof}

\begin{lemma}
\label{lem:bad_event_expectation}
For any $t \in [T]$ and $i \in \mathcal{A}$, we have

\begin{align*}
    \E{ \vert \theta^*_i - \hat{\theta}_{i,t-1} \vert - \sqrt{\frac{8 \log T}{N_{t-1}(i)}} \;\;\;\bigg\vert\;\;\; \vert \theta^*_i - \hat{\theta}_{i,t-1} \vert > \sqrt{\frac{8 \log T}{N_{t-1}(i)}}} \leq 1
\end{align*}
\end{lemma}
\begin{proof}

We know that the average feedback $\hat{\theta}_{i,t-1}$ is Gaussian with $\E{\hat{\theta}_{i,t-1}} = \theta^*_i$ and variance $\leq 1$. Let $Z := \theta^*_i - \hat{\theta}_{i,t-1}$. Then, $Z$ is Gaussian with mean $0$ and variance $\leq 1$. The following holds:

\begin{align*}
    \E{ \vert Z \vert - \sqrt{\frac{8 \log T}{N_{t-1}(i)}} \;\;\;\bigg\vert\;\;\; \vert Z \vert > \sqrt{\frac{8 \log T}{N_{t-1}(i)}}} =& \\
    \left(\E{ Z - \sqrt{\frac{8 \log T}{N_{t-1}(i)}} \;\;\;\bigg\vert\;\;\; Z > \sqrt{\frac{8 \log T}{N_{t-1}(i)}}} + \E{ (- Z) - \sqrt{\frac{8 \log T}{N_{t-1}(i)}} \;\;\;\bigg\vert\;\;\; (-Z) > \sqrt{\frac{8 \log T}{N_{t-1}(i)}}}\right)/2 =& \\
    \intertext{($Z$ is $0$-mean Gaussian, hence $Z \overset{d}{=} -Z$)}
    \E{ Z - \sqrt{\frac{8 \log T}{N_{t-1}(i)}} \;\;\;\bigg\vert\;\;\; Z > \sqrt{\frac{8 \log T}{N_{t-1}(i)}}} =& \\
    \E{ Z - \sqrt{\frac{8 \log T}{N_{t-1}(i)}} \;\;\;\bigg\vert\;\;\; Z - \sqrt{\frac{8 \log T}{N_{t-1}(i)}}> 0} &
\end{align*}

We notice that $(Z - \sqrt{\frac{8 \log T}{N_{t-1}(i)}})$ is Gaussian with mean $(- \sqrt{\frac{8 \log T}{N_{t-1}(i)}})$, where $(- \sqrt{\frac{8 \log T}{N_{t-1}(i)}}) < 0$. Furthermore (see e.g., Theorem 2 in \cite{horrace2015moments}), the expected value after truncation is increasing in $(- \sqrt{\frac{8 \log T}{N_{t-1}(i)}})$. Hence,

\begin{align*}
    \E{ Z - \sqrt{\frac{8 \log T}{N_{t-1}(i)}} \;\;\;\bigg\vert\;\;\; Z - \sqrt{\frac{8 \log T}{N_{t-1}(i)}}> 0} \leq& \\
    \E{ Z \;\;\;\bigg\vert\;\;\; Z > 0} \leq& \\
    \intertext{($Z$ is Gaussian with mean $0$ and variance $\leq 1$)}
    \phi(0) / (1 - \Phi(0)) \leq 1
\end{align*}
\end{proof}

\begin{theorem}
Given the optimal super arm $\bm{a^*}$ for Algorithm \ref{alg:exact_algorithm} and the optimal super arm $\Tilde{\bm{a}}^*$ for Algorithm \ref{alg:approximation_algorithm}, we have that $ f_{\bm{\theta}^*}(\Tilde{\bm{a}}^*) - f_{\bm{\theta}^*}(\bm{a}^*) \leq \sqrt{2 \log d}$.
\end{theorem}

\begin{proof}

For any super arm $\bm{a} \in \mathcal{I}$, let $Y_i$ for $i \in \bm{a}$ be Gaussian random variables with $Y_i \sim \mathcal{N}(\theta_i^*, \sigma_i^2)$. Let $W_i := Y_i - \theta_i^*$, such that $W_i \sim \mathcal{N}(0, \sigma_i^2)$. Then, the following holds:

\begin{align*}
    \mathbb{E}\left[ \max_{i\in \bm{a}} Y_i \right] =& \\ 
    \mathbb{E}\left[ \max_{i\in \bm{a}} (W_i + \theta_i^*) \right] \leq& \\
    \mathbb{E}\left[ \max_{i\in \bm{a}} \left(W_i + \max_{j\in \bm{a}} \mathbb{E}\left[ Y_j \right]\right) \right] =& \\
    \mathbb{E}\left[ \max_{i\in \bm{a}} W_i \right] + \max_{i\in \bm{a}} \mathbb{E}\left[ Y_i \right] \leq& \\
    \intertext{(By Lemma 9 in \cite{orabona2015optimal}, and $\sigma_i^2 \leq 1$ for all $i \in \bm{a}$)}
    \sqrt{2 \log d} + \max_{i\in \bm{a}} \mathbb{E}\left[ Y_i \right]
\end{align*}

Furthermore, by Jensen's inequality, we have $\max_{i\in \bm{a}} \mathbb{E}\left[ Y_i \right] \leq  \mathbb{E}\left[ \max_{i\in \bm{a}} Y_i \right]$. Moreover, by definition we know that $\bm{a}^* = \arg \min_{\bm{a} \in \mathcal{I}} \E{\max_{i \in \bm{a}} Y_i}$ and $\Tilde{\bm{a}}^* = \arg \min_{\bm{a} \in \mathcal{I}} \max_{i \in \bm{a}} \E{ Y_i}$. Consequently, we have:

\begin{align*}
    \max_{i \in \Tilde{\bm{a}}^*} \E{Y_i} \leq& \\
    \max_{i \in \bm{a}^*} \E{Y_i} \leq& \\
    \E{\max_{i \in \bm{a}^*} Y_i} \leq& \\
    \E{\max_{i \in \Tilde{\bm{a}}^*} Y_i} \leq& \\
    \sqrt{2 \log d} + \max_{i\in \Tilde{\bm{a}}^*} \mathbb{E}\left[ Y_i \right]
\end{align*}

Hence, we conclude

\begin{align*}
     f_{\bm{\theta}^*}(\Tilde{\bm{a}}^*) - f_{\bm{\theta}^*}(\bm{a}^*) = \E{\max_{i \in \Tilde{\bm{a}}^*} Y_i} - \E{\max_{i \in \bm{a}^*} Y_i} \leq \sqrt{2 \log d}
\end{align*}

\end{proof}

\bibliography{main}

\end{document}